\def\year{2020}\relax

\documentclass[letterpaper]{article} 
\usepackage{arxiv}  
\usepackage{times}
\usepackage{helvet}
\usepackage{courier}
\usepackage[hyphens]{url}
\usepackage{graphicx}
 \usepackage{hyperref}
\hypersetup{
colorlinks=true,
    citecolor = {blue}
}

\usepackage[utf8]{inputenc}
\usepackage{amsfonts}       
\usepackage{nicefrac}       
\usepackage{microtype}      

\usepackage{subfigure} 
\usepackage{amsmath}
\usepackage{amsthm}
\usepackage{amssymb}
\usepackage{fancyhdr}
\usepackage{mathrsfs}
\usepackage{epsfig}
\usepackage{graphics}
\usepackage{epsf}
\usepackage{multirow, paralist}
\usepackage{mathtools}
\usepackage{algorithm,algorithmicx}
\usepackage{algpseudocode}
\usepackage{array}

\newtheorem{lemma}{Lemma}
\newtheorem{thm}{Theorem}

\newtheorem{remark}{Remark}

\def \a {\mathbf{a}}

\def \u {\mathbf{u}}

\def \v {\mathbf{v}}

\def \w  {\mathbf{w}}

\def \x {\mathbf{x}}

\def \A {\mathbf{A}}

\def \I {\mathbf{I}}
\def \J {\mathbf{J}}

\def \S {\mathbf{S}}

\def \W {\mathbf{W}}

\def \bqsa {\begin{eqnarray}}
\def \eqsa {\end{eqnarray}}
\def \bqs {\begin{equation}\begin{aligned}}
\def \eqs {\end{aligned}\end{equation}}

\title{Distributed Primal-Dual Optimization for Online Multi-Task Learning}
\author{
Peng Yang and Ping Li \\
Cognitive Computing Lab\\
Baidu Research\\
10900 NE 8th ST, Bellevue WA, 98004, USA\\
\{pengyang01, liping11\}@baidu.com
}

\begin{document}

\maketitle

\begin{abstract}

Conventional online multi-task learning algorithms suffer from two critical limitations: 1) Heavy communication caused by delivering high velocity of sequential data to a central machine; 2) Expensive runtime complexity for building task relatedness. To address these issues, in this paper we consider a setting where multiple tasks are geographically located in different places, where one task can synchronize data with others to leverage knowledge of related tasks. Specifically, we propose an adaptive primal-dual algorithm, which not only captures task-specific noise in adversarial learning but also carries out a projection-free update with runtime efficiency.
Moreover, our model is well-suited to decentralized periodic-connected tasks as it allows the energy-starved or bandwidth-constraint tasks to postpone the update. Theoretical results demonstrate the convergence guarantee of our distributed algorithm with an optimal regret. Empirical results confirm that the proposed model is highly effective on various real-world datasets.

\end{abstract}

\section{Introduction}

Multi-task learning (MTL) is widely used learning framework where similar tasks are considered jointly for the purpose of improving performance compared to learning the tasks separately~\cite{caruana1997multitask}. By transferring information between tasks it is hoped that samples will be better utilized, leading to improved generalization performance.
MTL has been successfully applied in practical scenarios, e.g., speech recognition~\cite{seltzer2013multi}, image classification~\cite{lapin2014scalable}, disease gene prediction~\cite{zhou2013modeling}, etc. Recent years also witness extensive studies on streaming data, known as online multi-task learning (OMTL)~\cite{DekelLS06,Saha,yang2017robust}, for the merits of capturing the dynamically changing and uncertain nature of the environment, which is in contrast to the offline setting in which the objective functions are fixed~\cite{liu2017distributed,smith2017federated}.

Existing OMTL techniques suffer from the heavy communication caused by centralizing the high velocity of sequential data from different locations to a single machine. In this paper, we address OMTL problem in a distributed manner, i.e., multiple tasks are geographically located in different places, where task models can synchronize data with others to leverage knowledge of related tasks. In such setting, each task $i$ is endowed with a sequence of objective functions $(f^i_{t})_{t=1}^T$, where $f^i_{t}$ is the loss function of $i$-th task  at round $t$. The goal boils down to minimizing the sequential objective functions across $m$ different tasks,
\bqs\label{DMTL}
\min_{\W: \W \in \mathcal{K}} \sum_{t=1}^TF_t(\W) 
\eqs
where $F_t(\W) = \sum_{i=1}^m f^i_{t}(\w^i)$ is denoted as the instantaneous loss at round $t$, while $\W = [\w^1,\ldots,\w^m]\in\mathbb{R}^{d\times m}$ is parameter matrix for $m$ tasks. Note that $\mathcal{K}$ is a closed convex subset characterized by an inequality, i.e., $\mathcal{K} = \{\W | g(\W)\leq 0\}$. It aims to constrain $\W$ into simple sets, e.g., hyperplanes, balls, bound constraints, etc.  We assume $m \leq d$ without loss of generality. The algorithm for solving problem~(\ref{DMTL}) has to be distributed, in the sense that each task, accessing only local task data, is required to communicate with other tasks in a bandwidth-limited network.
Moreover, the distributed tasks have to learn incrementally from data streams, for the merits of capturing the dynamic nature of the changing environment~\cite{zhang2018distributed}.

The success of multitask learning relies on the relatedness between tasks. To build the task relationship, existing algorithms used low-rank matrices to enforce different tasks to share a common structure~\cite{smith2017federated,yang2019confidence}. The low-rank optimization problem can be solved by the first-order optimization methods, e.g., subgradient descent~\cite{bauschke2011fixed}. 
Although these methods are guaranteed to converge, they are inefficient because a singular value decomposition (SVD), which takes $O(dm^2)$ time, is required at each round. To reduce the computation complexity, many efficient solvers have been developed by replacing the full SVD with a partial SVD. However, those approaches either require the function to be smooth~\cite{wang2016distributed1} or are designed for the constraint optimization~\cite{zheng2018distributed}.
Furthermore, the problem that different tasks generally have different noise levels was ignored by those approaches mentioned above.
The adversarial noise with large loss residues may corrupt the task relativeness by the conventional convex loss functions.
To deal with noise, a calibrated multivariate regression approach was developed in~\cite{liu2014multivariate}, and then it was further improved in~\cite{gong2014efficient}. Nevertheless, both of them are based on feature learning and the optimization techniques are computationally expensive.

In this work, we propose an efficient distributed algorithm to address both issues simultaneously.  The main contributions of this work are summarized as follows:
\begin{enumerate}
\item We introduce a capped $L_p$-norm loss function to capture the adversarial noise. We derive a weighted loss function to iteratively reduce the negative impact of noise according to noise level of specific tasks.
\item The constrained task relatedness is learned by a projection free primal-dual algorithm. In each round, it only needs to compute the leading singular vectors instead of a full SVD, reducing time complexity from $O(dm^2)$ to $O(dm)$.
\item The proposed algorithm is well-suited to decentralized periodic-connected tasks, as it allows the energy-starved or bandwidth-limited tasks to alleviate synchronization delay.
\item Theoretical results demonstrate the convergence guarantee of  our distributed model with an optimal regret. Empirical results
confirm that the proposed algorithm is  effective.
\end{enumerate}

\section{Algorithm}

In this problem, we are faced with $m$ different but related classification problems also known as tasks. The task model is learned on a sequence of instance-label pairs, i.e., $\{(\x_t^i, y_t^i)\}_{1\leq t\leq T}^{1\leq i\leq m}$, where the instance $\x_t^i\in\mathbb{R}^{d}$ is drawn from a distinct distribution $p^i$, and $y_t^i\in\{\pm 1\}$. The algorithm maintains $m$ separate models in parallel, one for each task. When the instances $\{\x_t^1,\ldots,\x_t^m\}$ are observed at round $t$, the model generates a decision matrix $\W_t = [\w_t^1,\ldots,\w_t^m]\in\mathbb{R}^{d\times m}$ under a constraint set $\mathcal{K}$. Then it suffers the corresponding loss $F_t(\W_t) = \sum_{i=1}^m f^i_{t}(\w_{t}^i)$ where $f^i_{t}$ is a convex loss function. The goal of online learner is to generate a sequence of decision points $\{\W_{t}\}_{t=1}^T$, so that the regret regarding to the best fixed decision can be minimized,
\bqs\label{regret}
\hbox{Reg}_{T} := \sum_{t=1}^T F_t(\W_t) - \sum_{t=1}^T F_t(\W^*),
\eqs
where $\W^* = \operatorname*{argmin}_{\W\in\mathcal{K}} \sum_{t=1}^TF_{t}(\W)$ is the best decision in hindsight. An algorithm achieves nontrivial performance if its regret is sublinear over the number of total rounds $T$.

The success of multi-task learning relies on the relatedness between tasks. To learn the task relationship, existing algorithms exploit low-rank constraints to enforce different tasks to share a common structure~\cite{smith2017federated,xie2017privacy,baytas2016asynchronous}.
To yield a low-rank solution in (\ref{regret}), these methods aim to minimize the following constrained problem:
\bqs\notag
\min_{\W} \sum_{t=1}^T\sum_{i=1}^m f^i_t(\w^i), \ \ \hbox{s.t.} \ \ \hbox{rank}(\W)\leq r
\eqs
where $r$ is a predefined value with $r \ll \min(d,m)$, and $\hbox{rank}(\cdot)$ denotes the matrix rank, i.e., the number of non-zero singular values.
Note that the constrained objective is equivalent to the regularized objective function with a proper parameter $\lambda > 0$,
\bqs\label{W-regular}
\min_{\W} \sum_{t=1}^T\sum_{i=1}^m f^i_{t}(\w^i)+ \lambda\ \hbox{rank}(\W).
\eqs
Although above problems are equivalent, specific optimization techniques could be more suitable for one particular type of objective functions\footnote{ Alternating direction method of multipliers for regularized objective function and frank-wolfe for constrained objective function. Gradient descent methods can be adopted for both, leads to proximal and projected methods, respectively.}.
For convenience, we won't distinguish between these two formulations in this work.\\

In this paper, we make the following assumptions: 
\begin{itemize}
  \item The loss function $f_t(\w)$ is convex, i.e., $\forall \ \w, \w'$ in the domain of $f_t$,
$f_t(\w) \geq f_t(\w') + \langle \nabla f_t(\w') , \w - \w' \rangle$. 
  \item The loss function $f_{t}(\w)$ is $\beta$-\emph{Lipschitz} on a convex set, i.e., $\forall \ \w, \w'$ in the domain of $f_t$, $|f_{t}(\w) - f_{t}(\w')| \leq \beta\|\w - \w'\|_2$. 
  \item The concave function $h(u) = \min(u^p, \xi)\ (\xi > 0)$ has a bounded supergradient at any point $u = f(\w)$ with $p\in(0,1)$, i.e., $ \| \nabla_{u} h(u) \|_2 \leq \kappa $.
  \item Euclidean diameter of primal variable $\w$ or dual variable $\a$ is bounded by $D$, i.e., $\|\w - \w'\|_2 \leq D$ and $\|\a - \a'\|_2 \leq D$.
\end{itemize}

\subsection{Adversarial Learning}



In adversarial learning, the feedback observed by the learner is malicious inputs designed to fool machine learning models. 
Adversarial noise with large loss residues may corrupt task relativeness by the conventional convex loss $f_t$.
To be resistant to noise, each task should have a specific regularization parameter that depends on the specific noise level. To achieve this goal, we exploit a capped $L_p$-norm function with $p\in(0,1)$:
\bqs\label{L_P-norm}
  h(f^i_t(\w^{i})) = \min\left(f^i_t(\w^{i})^p, \xi\right),
\eqs
where $h(f^i_t(\cdot))$ enforces a capped $L_p$-norm over loss function $f^i_t$ with an upper bound $\xi > 0$. It indicates that no matter how misclassified the data point is, the loss residue in (\ref{L_P-norm}) is capped by $\xi$. This makes the loss function robust to noise since their effect to the model is bounded. However, optimizing this problem is difficult since the function $\min(f^i_t(\cdot)^p, \xi)$ is a concave non-smooth function in the domain of $f^i_t$. Motivated by concave duality~\cite{rockafellar1970convex}, Lemma~\ref{reweight-obj} provides an iterative weighted function to solve it.

\begin{lemma}\label{reweight-obj}
Problem~(\ref{L_P-norm}) can be relaxed to minimizing a weighted convex formulation:
\bqs\label{rew-rule}
 \min_{\w^i} \gamma_t^i f^i_{t}(\w^i),
\eqs
where $\gamma_t^i = \nabla_uh(u)|_{u=f^i_t(\w)}$ is the supergradient of the concave function $h(u)$ at the point $u=f^i_{t}(\w_t^i)$,
\bqs\label{gamma}
\gamma_t^i = \left\{
                  \begin{array}{ll}
                    p f^i_{t}(\w_t^i)^{p-1}, & f^i_{t}(\w_t^i)^{p} \leq \xi \\
                    0, & \hbox{otherwise}
                  \end{array}
                \right.
\eqs
\end{lemma}

\begin{proof}
Let $u = f(\cdot)$ and $h(u) = \min(u^p, \xi)$. Since $f(\cdot)$ is a convex function, $h(u)$ can be formulated as:
\bqs\label{dual-loss}
\min(u^p, \xi) = \inf_{\gamma\geq 0}[\gamma u - h^*(\gamma)],
\eqs
where $h^*(\gamma)$ is the concave dual of $h(u)$, defined as:
\bqs\notag
h^*(\gamma) & = \inf_{u>0}(\gamma u - h(u)) \overset{(\ref{dual-loss})} {=} \inf_{u>0}[\gamma u - \min(u^p, \xi)] \\
= & \left\{
           \begin{array}{ll}
             \frac{p-1}{p}p^{\frac{1}{1-p}}\gamma^{\frac{1-p}{p}}, & \hbox{if} \ u^p < \xi \\
             \gamma\xi^{\frac{1}{p}} - \xi, & \hbox{if}\ u^p \geq \xi.
           \end{array}
  \right.
\eqs
Equipped $h^*(\gamma)$ back to Eq.~(\ref{dual-loss}), we obtain $\gamma$ as in Eq.~(\ref{gamma}).
\end{proof}

\noindent 
\begin{remark}
\emph{We observe that $\gamma_t^i$ depends on the loss $f^i_{t}(\w_t^i)$. In particular, a misclassified point with $f^i_{t}(\w_t^i)^p > \xi$ will be considered as an outlier, and ignored, i.e., $\gamma_t^i = 0$.}
\end{remark}

\begin{remark}
\emph{The derived solution in~(\ref{rew-rule}) minimizes the upper bound of the concave problem~(\ref{L_P-norm}) iteratively.
Since $h(f_t(\w^i))$ is a concave function, for any $\w^i$ we obtain an upper bound of $h(f^i_t(\w^i))$ via a linear approximation,
\bqs\notag
h(f^i_t(\w^i)) \leq h(f^i_t(\w^i_t)) + \langle \gamma^i_{t}, f^i_t(\w^i) - f^i_t(\w^i_t) \rangle,
\eqs
where $\gamma^i_{t} = \nabla_uh(u)|_{u=f^i_t(\w^i_t)}$. 
Since $f_t(\w_t^i)$ is constant, 
$ \min_{\w^i} \ h(f^i_t(\w_t^i)) + \langle \gamma_t^i, f^i_t(\w^i) - f^i_t(\w_t^i) \rangle \equiv \min_{\w^i} \ \langle \gamma_t^i, f^i_t(\w^i)\rangle$, which obtains a convex loss to minimize the upper bound of $h(f^i_t(\w^i))$.}
\end{remark}

\subsection{Projection-free Optimization}

The refined problem, $\sum_{t=1}^T\sum_{i=1}^m \langle \gamma_t^i, f^i_t(\w^i)\rangle + \lambda\hbox{rank}(\W)$, is non-convex and computationally intractable~\cite{amaldi1998approximability}. We relax $\hbox{rank}(\cdot)$ to its convex surrogate, i.e., nuclear norm $\|\cdot\|_*$, then  the problem becomes
\bqs\label{W-regular-1}
\min_{\W} \sum_{t=1}^T\sum_{i=1}^m \langle \gamma_t^i, f^i_t(\w^i)\rangle + \lambda\|\W\|_*.
\eqs
The nuclear norm minimization can be solved by gradient descent and proximal gradient descent. Although these methods are guaranteed to converge, they have to perform a full SVD of $\W_t$ in each round, which suffers a hight runtime complexity of $O(dm^2)$. \cite{hazan2012projection} provided a linear optimization method to solve this issue, but its computational effectiveness is achieved at the expense of a suboptimal regret. 
To reduce runtime complexity, we study the dual form of the nuclear norm, $\|\W\|_* = \max_{\|\A\|_2 \leq 1} \hbox{tr}(\A^{\top}\W)$ where $\|\cdot\|_2$ is the spectral norm, and then cast the problem~(\ref{W-regular-1}) into the following primal-dual formulation:
\bqs\notag
\min_{\W}\max_{\A}  \sum_{t=1}^T\sum_{i=1}^m \langle \gamma_t^i, f^i_t(\w^i)\rangle + \lambda\hbox{tr}(\A^{\top}\W)  \ \hbox{s.t.} \ \  \|\A\|_2 \leq 1.
\eqs
Since the above optimization problem is convex-concave, we can apply the online subgradient method to solve it.
However, due to the spectral norm constraint of $\A$, we have to project the intermediate solution onto the unit spectral norm ball, which again requires a full SVD operation~\cite{xiao2017svd}.

To address this issue, we replace the constraint $\|\A\|_2 \leq 1$ with a regularization term to control the spectral norm of $\A$,
\bqs\notag
\min_{\W}\max_{\A} \sum_{t=1}^T\sum_{i=1}^m \gamma_t^if^i_t(\w^i) + \lambda\hbox{tr}(\A^{\top}\W) - \rho[\|\A\|_2 - 1]_+,
\eqs
where $\rho > 0$ is a trade-off parameter and $[\cdot]_+ = \max(0,\cdot)$. We assign $\rho = 1$, $\lambda = 1$ since such setting can control the rank, i.e., $\|\W\|_* \leq \rho/\lambda$. To solve the above problem, we can use the online subgradient method~\cite{shalev2012online}, which iterates as follows:
\bqs\notag
&\A_{t+1} = \A_t + \eta_t\left(\W_{t} - \partial[\|\A_{t}\|_2 - 1]_+\right), \\
&\W_{t+1} = \W_t - \eta_t\left(\A_{t+1} + \nabla F_{t}(\W_t)\Gamma_t\right),
\eqs
where $\Gamma_t = \hbox{diag}(\gamma_t^1,\ldots,\gamma_t^m)\in\mathbb{R}^{m\times m}$.

Note that the subgradient $\partial[\|\A\|_2 - 1]_+$ can be computed efficiently. We denote $\sigma_1(\A)$ as the leading singular value of $\A$, $\u$ and $\v$ as the corresponding left and right singular vectors. Then we have
\bqs\notag
\u\v^{\top}\mathbb{I}(\sigma_1(\A) > 1) \in \partial[\|\A\|_2 - 1]_+.
\eqs
In each round, we only need to compute the leading singular vector of $\A_t$ with $O(dm)$ time. In contrast, a full SVD takes $O(dm^2)$ time.

\begin{table*}
\centering
\caption{Comparison with distributed variants of proximal gradient descent (ProxGD), alternating direction method of multipliers (ADMM), online frank-wolfe (OFW) in terms of computational complexity and communication efficiency.}
\label{Summary-Resources-Required}
\begin{tabular}[3\textwidth]{|c|c|c|c|c|c|} \hline
  Algorithms       & Worker Comp.       &  Server Comp.         & Communication   & Time Complexity    &  Regret  \\ \hline \hline
 ProxGD            & Gradient Comp.     &  SV Shrinkage         &   $2d$    & $m^2d$              &  $T^{1/2}$ \\ \hline
 ADMM              & ERM                &  SV Shrinkage         &   $3d$    & $m^2d$              &  $T^{1/2}$ \\ \hline
 OFW               & Gradient Comp.     &  Leading SV Comp.     &   $2d$    & $md$                &  $T^{3/4}$ \\ \hline \hline
 DROM              & Gradient Comp.     &  Leading SV Comp.     &   $2d$    & $md$                &  $T^{1/2}$ \\ \hline
\end{tabular}
\end{table*}

\section{Distributed Learning}

Though above algorithm is efficient, heavy communication is caused by centralizing the high velocity of sequential data to a central machine. To address this issue, we show how to perform the primal-dual optimization in a distributed manner.

Assume that tasks are distributed on local worker machines, i.e., one worker for each task, our core idea is to solve the \emph{local problem} on each local machine independently, and then centralize the updated information of each task to efficiently solve a \emph{central problem}.
The proposed algorithm, namely DROM, is summarized in Algorithm~\ref{POOM}. It runs an alternating optimization procedure that comprises two steps: 1) Local-step: solving $\{\w^i, \a^i\}_{i=1}^m$ in a distributed manner among local workers independently; 2) Central-step: solving $\partial[\|\A\|_2 - 1]_+$ with aggregated $\{\a^i\}_{i=1}^m$ from all workers on central server. Figure~\ref{master-slaver-setting} illustrates the  procedure~of~DROM.

\begin{algorithm}[h!]
\caption{DROM: Distributed Primal-dual optimization for Online MTL}\label{POOM}
\begin{algorithmic}[1]
\State {\bf Input}: data $\{\x^i_{t},y^i_{t}\}$ with $i\in[m]$ and $t\in[T]$ distributed over $m$ machines, parameters $p$ and $\xi$ \;
\State {\bf Initialize}: $\w^i_{0} = \mathbf{0}$, $\a^i_{0} = \mathbf{0}$ for all workers $i\in[m]$ \;
\For{$t=1,\ldots, T$}
  \State \textbf{for all workers (Local-step) do in parallel}
  \State \quad $
         \gamma_t^i = \left\{
                  \begin{array}{ll}
                    p f^i_{t}(\w_t^i)^{p-1}, & f^i_{t}(\w_t^i)^{p} \leq \xi \\
                    0, & \hbox{otherwise}
                  \end{array}
                \right.
         $
        \State \quad When $\gamma_t^i > 0$, local update with $\eta_t = \frac{1}{\sqrt{t}}$:
        \bqs\label{primal-dual-optim-local}
         & \a^i_{t+1} = \a^i_{t} + \eta_t \left(\w^i_{t} - [\u\v^{\top}]_i \right); \\
         & \w^i_{t+1} = \w^i_{t} - \eta_t \left(\a^i_{t+1} + \gamma^i_{t}\nabla f^i_{t}(\w^i_t) \right);
         \eqs
         \State \quad send $\a^i_{t+1}$ to the server, wait to receive $\{\u,[\v]_i\}$;
  \State  \textbf{Reduce (Central-step)}: The server aggregates $\A$ and computes
        $\u\v^{\top}\mathbb{I}(\sigma_1(\A) > 1) \in \partial[\|\A\|_2 - 1]_+;
        $
  \State Server sends back $\{\u,\v\}$ when $\sigma_1(\A) > 1$;
\EndFor
\State {\bf Output}: $\W_T$ \;
\end{algorithmic}
\end{algorithm}

\begin{figure}[h!]
\centering
\includegraphics[width=0.95\columnwidth]{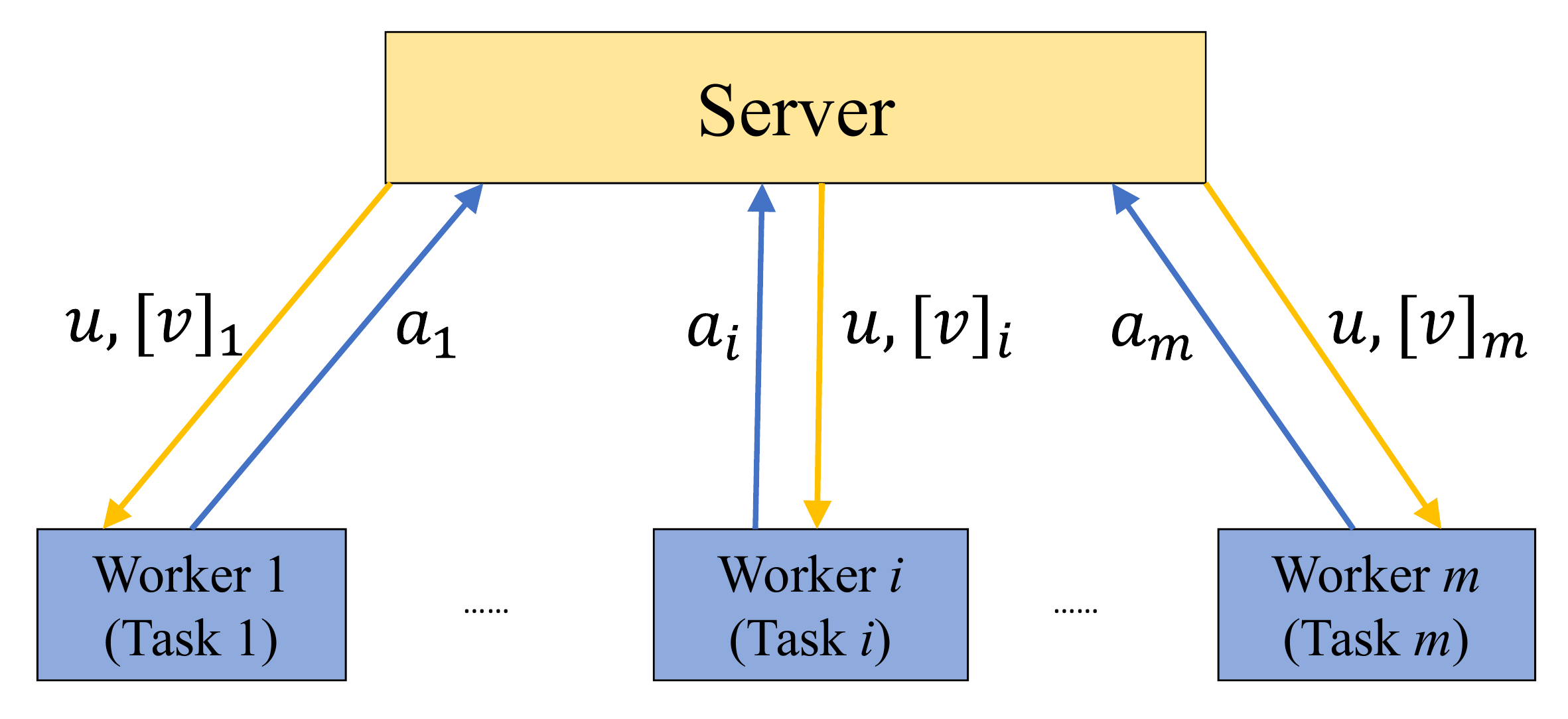}
\vspace{-0.1in}
\caption{Distributed Primal-Dual Optimization}\label{master-slaver-setting}
\end{figure}

Here we elaborate on the details of DROM:
\begin{itemize}
  \item \textbf{Model Variables}: At round $t$, the $m$ workers have primal variables $\{\w^1_t,...,\w^m_t\}$ for $m$ tasks, respectively. In addition, $m$ dual variables $\{\a^1_t,...,\a^m_t\}$ are stored at each worker. Each worker accesses local task data and updates variables independently.

  \item \textbf{Local Update}: The primal-dual optimization is conducted in a distributed manner. At round $t$, each worker $i$ is assigned a problem that accesses local data $(\x^i_t, y^i_t)$. Local problem is solved in two steps: 1) Computing the weight $\gamma_t^i$ based on $f^i_{t}(\w^i_t)$; 2) Performing an alternative learning procedure as in~(\ref{primal-dual-optim-local}): optimizing $\a^i$ with vector $[\u\v^{\top}]_i$; optimizing $\w^i$ with gradient $\nabla f^i_{t}(\w^i_t)$. Note that synchronization with the server is not allowed if noise or outlier is identified, i.e., $\gamma_t^i = 0$.

  \item \textbf{Central Update}: When the local update ends, the worker $i$ sends $\a^i_{t+1}$ to the central server. As we know that $\u\v^{\top}\mathbb{I}(\sigma_{1}(\A) > 1)\in \partial[\|\A\|_2 - 1]_+$, the server aggregates the local updates on $\A$ from all workers to calculate $(\u,\v)$, and then sends back $\u[\v]_i$ to the corresponding worker $i$. It is efficient for central computing with $O(dm)$ time. Note that the server sends back $(\u,\v)$ only when the corresponding spectrum $\sigma_1(\A) > 1$, which alleviates communication cost per round.

\end{itemize}

Motivated by the analysis in~\cite{xiao2017svd}, we provide the theoretical guarantee for this distributed algorithm regarding the regret.
The regret is based on the function $f$. Recall that minimizing the (\ref{rew-rule}) with $f(\cdot)$ is to minimize the upper bound of the (\ref{L_P-norm}) with $h(f(\cdot))$. When $f(\cdot)$ is converged to the optimal points, it infers an optimal upper bound for $h(f(\cdot))$.

\begin{thm}\label{thm-1} 
For all $t > 1$, the algorithm DROM runs over arbitrary instance-label pairs $\{(\x_t^i, y_t^i)\}_{i=1}^m$ with the update rule~(\ref{primal-dual-optim-local}). Assume $\A^*_t = \operatorname*{argmax}_{\|\A\|_2\leq 1}\hbox{tr}(\A^{\top}\W_t)$  and $\|\W_t\|_* \leq \rho / \lambda$ satisfied at all $t>1$. When $\eta_t = 1/\sqrt{t}$ the following regret is hold
\bqs\notag
\hbox{R}_T \leq m\sqrt{T}\left( D^2 + (\kappa\beta + \lambda D)^2 + (\rho + \lambda D)^2 \right).
\eqs
\end{thm}

\begin{remark}
The above theorem implies that the proposed algorithm is in the order of $O(\sqrt{T})$. This order is optimal, since the objective function is not strongly convex.
\end{remark}

\vspace{0.1in}

Table~\ref{Summary-Resources-Required} compares DROM with state-of-the-art baselines, e.g., Proximal Gradient Descent (ProxGD)~\cite{duchi2010composite}, Alternating Direction Method of Multipliers (ADMM)~\cite{boyd2011distributed} and Online Frank-Wolfe (OFW)~\cite{hazan2012projection}, in terms of runtime complexity and regret. From that table, we observe that DROM achieves a lower computational complexity with an optimal regret. Note that our method is different from OFW, since DROM directly constrains model parameters while OFW constrains the gradient descent. This work is different from ProxGD as well, since DROM is primal-dual algorithm while ProxGD optimizes only the primal variable.

\section{Decentralized Periodic Communication}


The algorithm DROM requires a central server to synchronize with all workers. This limits its applications in practical scenarios where each worker can only connect with its local neighbors in a bandwidth-limited network. 

For this reason, we propose a variant of DROM for decentralized periodic-connected tasks and summarize the whole procedure, namely DROM-D, in Algorithm~\ref{D-PDML}. This algorithm is parameterized by $\mathcal{P}(\S, \tau)$, where $\S\in\mathbb{R}^{m\times m}$ is an adjacency matrix used for inter-worker communication, and $\tau > 0$ is a synchronous interval for periodic update. These parameters improve the communication-efficiency in three different ways:

\begin{itemize}

\item \textbf{Group Synchronization}: The learning process does not rely on a fusion center or network-wide communication. Instead of synchronizing with all workers, a local worker just needs to exchange information with its neighbors, where the network topology is captured by the weight matrix $\S$. Therefore, using a sparse weight matrix $\S$ reduces the overall communication cost per round. Specifically, each worker $i$ becomes a ``local server'', and aggregates $\A^{(i)}$ from its neighbors,
    \bqs\notag
    \A^{(i)} = \A\times \hbox{Diag}([\S]_i), \ \ \ [\A^{(i)}]_j = \left\{
                                                                    \begin{array}{ll}
                                                                    \a^{j}, & \  j\in\mathcal{N}_i  \\
                                                                    \mathbf{0}, & \ j\not\in\mathcal{N}_i
                                                                    \end{array}
                                                                    \right.
    \eqs
    with $\mathcal{N}_i = \{j \ | \ S_{ij} = 1\}$ as the neighbors of the worker $i$. 
\item \textbf{Periodic Optimization}: The synchronization delay time is amortized over $\tau$ synchronous interval and is $\tau$ times smaller than fully synchronous update. Moreover, periodic optimization alleviates the synchronization delay in waiting for slow workers. Observe in Figure~\ref{decentralize_setting} that the idle time of workers is significantly reduced.
    To capture the synchronous interval, we use a time-varying matrix $\S_t$ that varies as:
    \bqs\notag
    \S_t = \left\{
              \begin{array}{ll}
                \S,           &  (t \ \hbox{mod} \ \tau) = 0 \\
                \I_{m\times m}, & \hbox{otherwise}
              \end{array}
         \right.
    \eqs
    where the identity matrix $\I_{m\times m}$ means that there is no iter-worker communication during the $\tau$ local updates.

\begin{figure}[t]
\includegraphics[width=1\columnwidth]{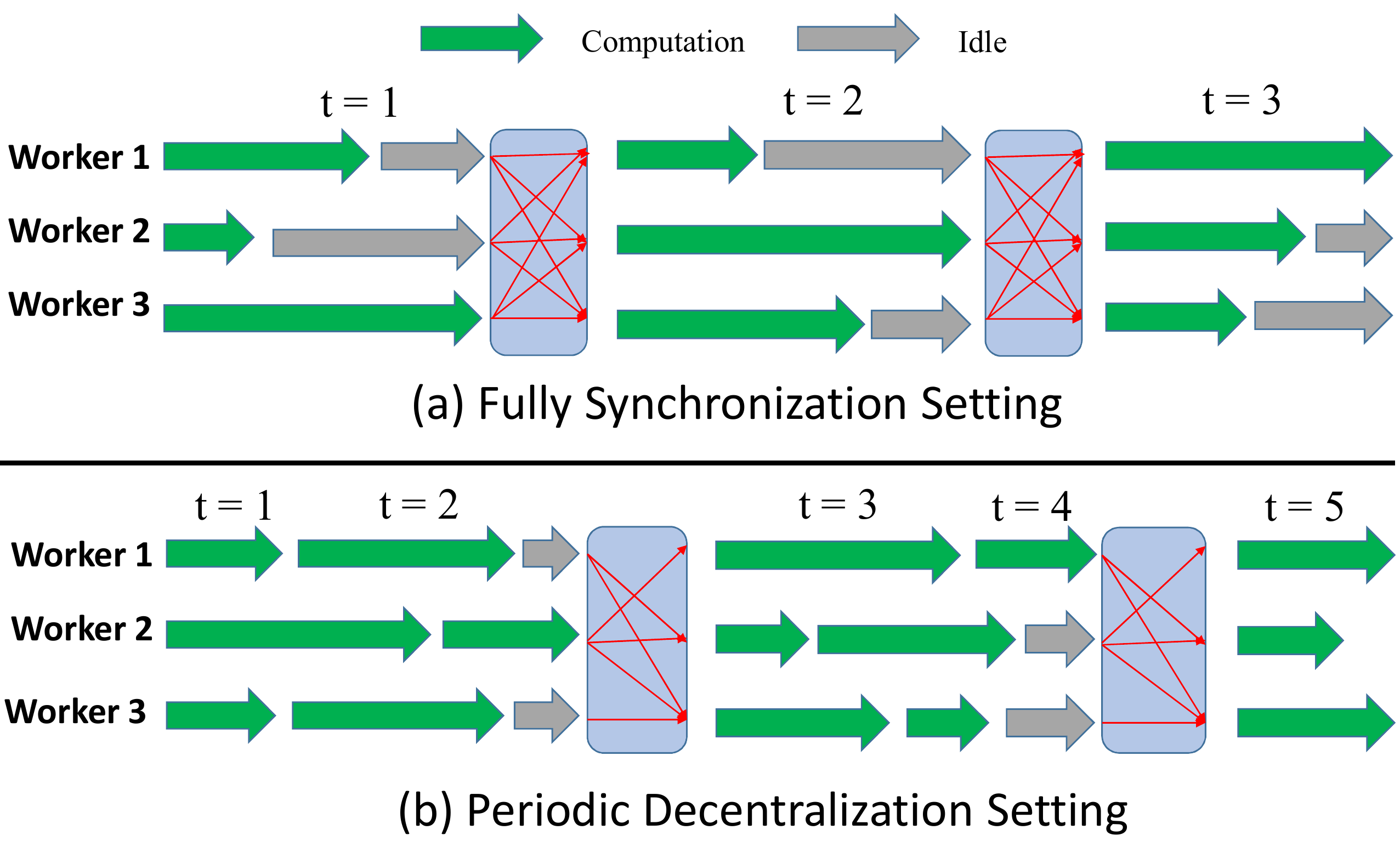}

\vspace{-0.1in}

\caption{ Illustration of communication-reduction strategies for $\tau = 2$. Green, red, grey arrows represent gradient
computation, communication, and idle state, respectively.}\label{decentralize_setting}
\end{figure}

\item \textbf{Non-blocking Execution}: As local update of $\A$ does not learn the gradient, the singular vector $\{\u, \v\}$ remains the same while worker nodes conduct local updates, i.e., $\{\u, \v\}_{t} = \{\u, \v\}_{t-1} = \ldots = \{\u, \v\}_{t-\tau+1}$ for $(t\ \hbox{mod}\ \tau) = 0$. Note that the workers only need $\{\u, \v\}_{t-\tau+1}$ before dual variable is updated from $\A_{t}$ to $\A_{t+1}$. Thus, there is no synchronous update until the workers perform  next $\tau$ rounds of local updates, which~reduces~synchronization~delay.
\end{itemize}


\begin{remark}
\emph{We study the update rule for existing synchronized algorithms since full synchronous algorithm corresponds to the special case $\S = \J = \mathbf{1}\mathbf{1}^{\top}, \tau = 1$. We show how existing communication-efficient algorithms are special cases of the general decentralized framework $\mathcal{P}(\S, \tau)$:}
\begin{itemize}
  \item Fully  Synchronization $\mathcal{P}(\J, 1)$: \emph{The local models are synchronized with all other workers after every round.}
  \item Periodic Synchronization $\mathcal{P}(\J, \tau)$: \emph{The local models are synchronized with all other workers after every $\tau$ rounds.}
  \item Periodic decentralization $\mathcal{P}(\S, \tau)$: \emph{The matrix $\S$ is fixed as a sparse weight matrix. Local model is updated via aggregating with few neighbors after every $\tau$ rounds.}
\end{itemize}
\end{remark}

Below provides theoretical guarantee of the decentralized periodic algorithm DROM-D regarding the regret.

\begin{thm}\label{thm-2}
The algorithm DROM-D runs over arbitrary sequential instance-label pairs.
Assume that $\tau \geq 1$ and $\S\in\mathbb{R}^{m\times m}$ is a random matrix with $S_{ij} \in[0,1]$. Let $\A^*_t = \operatorname*{argmax}_{\|\A\|_2\leq 1}\hbox{tr}(\A^{\top}\W_t)$  and $\|\W_t\|_* \leq \rho / \lambda$ are satisfied on any $t>1$. When $\eta_t = 1/\sqrt{\lceil t/\tau\rceil}$, the regret holds,
\bqs\notag
\hbox{R}_T \leq \sqrt{T}m \tau^{3/2} \left((D/\tau)^2 + (\kappa\beta + \lambda D)^2 + (\lambda D + \rho)^2\right).
\eqs
\end{thm}

\vspace{0.1in}

\begin{remark}
The regret is affected by the parameters $\lambda$ and $\rho$ that are related to task structure since the regret is hold when $\|\W\|_* \leq \rho/\lambda$. Assume that $D \leq 1$. If all tasks are identical, we have $\|\W\|_* = 1$, then regret becomes $\mathcal{O}(m\sqrt{T}\lambda^2\tau^{3/2})$ due to $\rho = \lambda$. If tasks are independent and unrelated with others, i.e., $\|\W\|_* = m$ leads to $\rho = m\lambda$, then regret becomes $O(m^3\sqrt{T}\lambda^2\tau^{3/2})$. It infers that a low-rank task structure yields to a small regret.
\end{remark}

\begin{algorithm}[t]
\caption{DROM-D: The DROM algorithm in Decentralized Periodic setting}\label{D-PDML}
\begin{algorithmic}[1]
\State {\bf Input}: $\{\x^i_{t},y^i_{t}\}$ with $i\in[m]$ and $t\in[T]$, the metrics $\mathcal{P}(\S,\tau)$, parameters $p$ and $\xi$ \;
\State {\bf Initialize}: $\w^i_{0} = \mathbf{0}$, $\a^i_{0} = \mathbf{0}$ for all workers $i\in[m]$ \;
\For{$t=1,\ldots, T$}
    \State \textbf{for all workers: $i=1,\ldots,m$ in parallel do}
        \State \hspace{0.2in} Solve \textbf{local problem} with $\eta_t = \frac{1}{\sqrt{\lceil t/\tau\rceil}}$:
        \bqs\notag
         & \gamma_t^i = \left\{
                  \begin{array}{ll}
                    p f^i_{t}(\w_t^i)^{p-1}, & f^i_{t}(\w_t^i)^{p} \leq \xi \\
                    0, & \hbox{otherwise}
                  \end{array}
                \right. \\
         & \w^i_{t+1} = \w^i_{t} - \eta_t \left(\a^i_{t} + \gamma^i_{t}\nabla f^i_{t}(\w^i_t) \right); \\
         & \a^i_{t+1} = \a^i_{t} + \eta_t \left(\w^i_{t+1} - [\u\v^{\top}]_i \right);
         \eqs
        \State  \textbf{If $t \ \hbox{mod} \ \tau = 0$ do central problem}:
            \State\hspace{0.2in} Broadcast $\a^i_{t+1}$ to its neighbors;
            \State\hspace{0.2in}  Wait to receive $\a^j_{t+1}$ from task $j\in \mathcal{N}_i$;
            \State\hspace{0.2in}  Aggregate $\A^{(i)} = \A_{t+1}\times \hbox{Diag}([\S]_i)$:
            \State\hspace{0.2in}  $\u\v^{\top} \mathbb{I}(\sigma_1(\A^{(i)}) > 1) \in \partial[\|\A^{(i)}\|_2 - 1]_+$;
\EndFor
\State {\bf Output}: $\W_T$ \;
\end{algorithmic}
\end{algorithm}

%
%
%

\begin{figure*}[t]
\begin{center}
\mbox{
\includegraphics[width=2.275in]{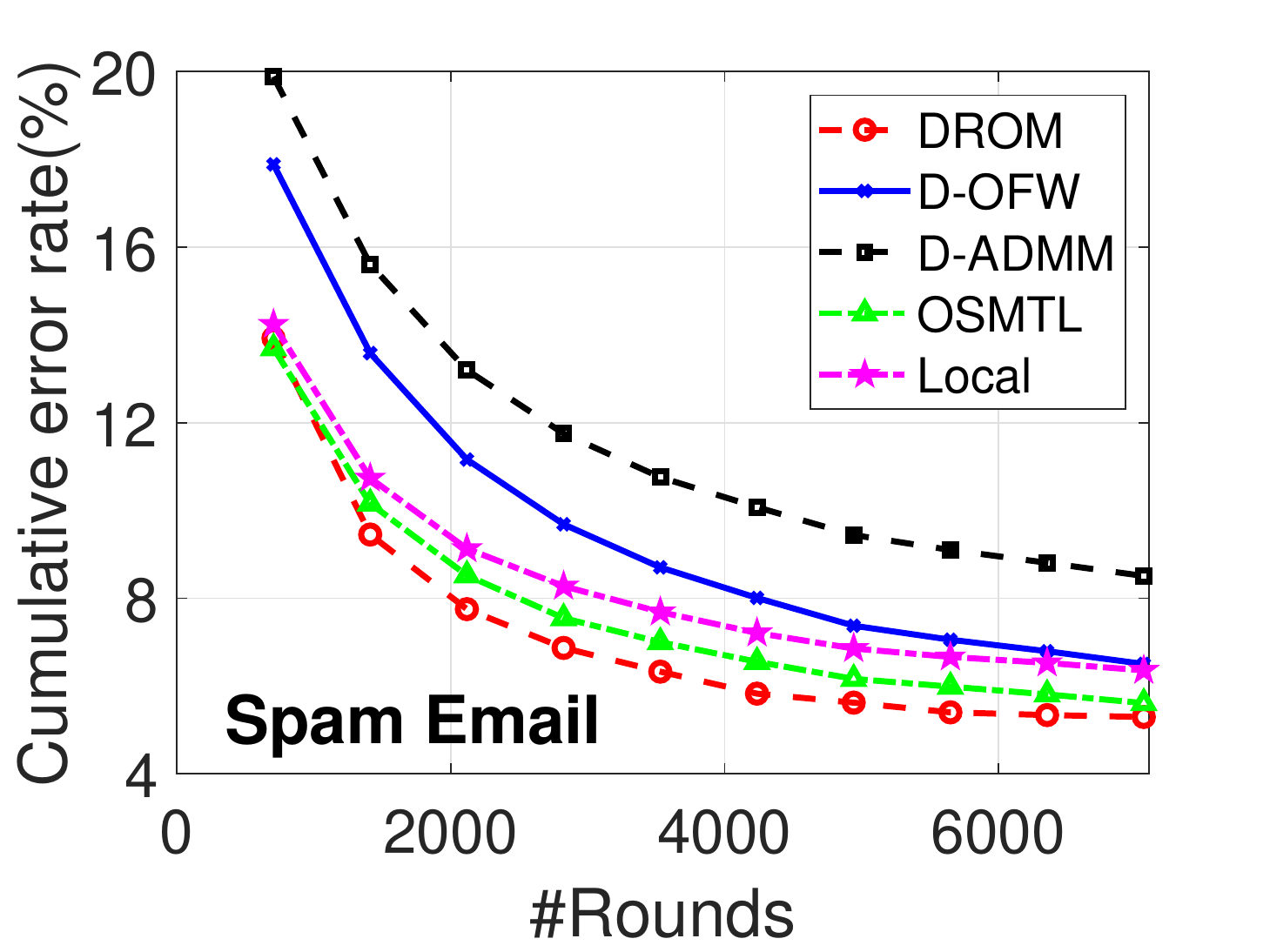}
\includegraphics[width=2.275in]{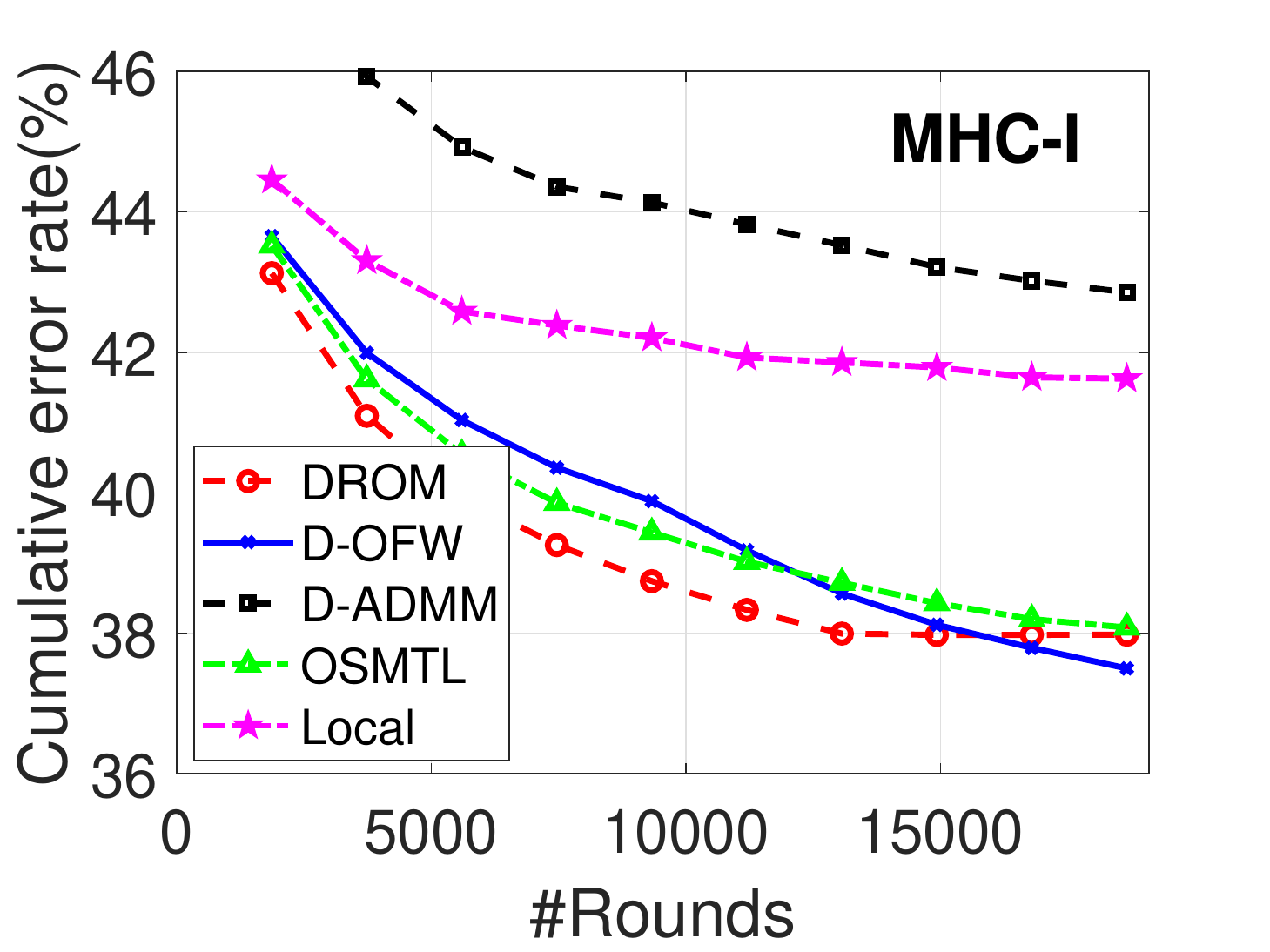}
\includegraphics[width=2.275in]{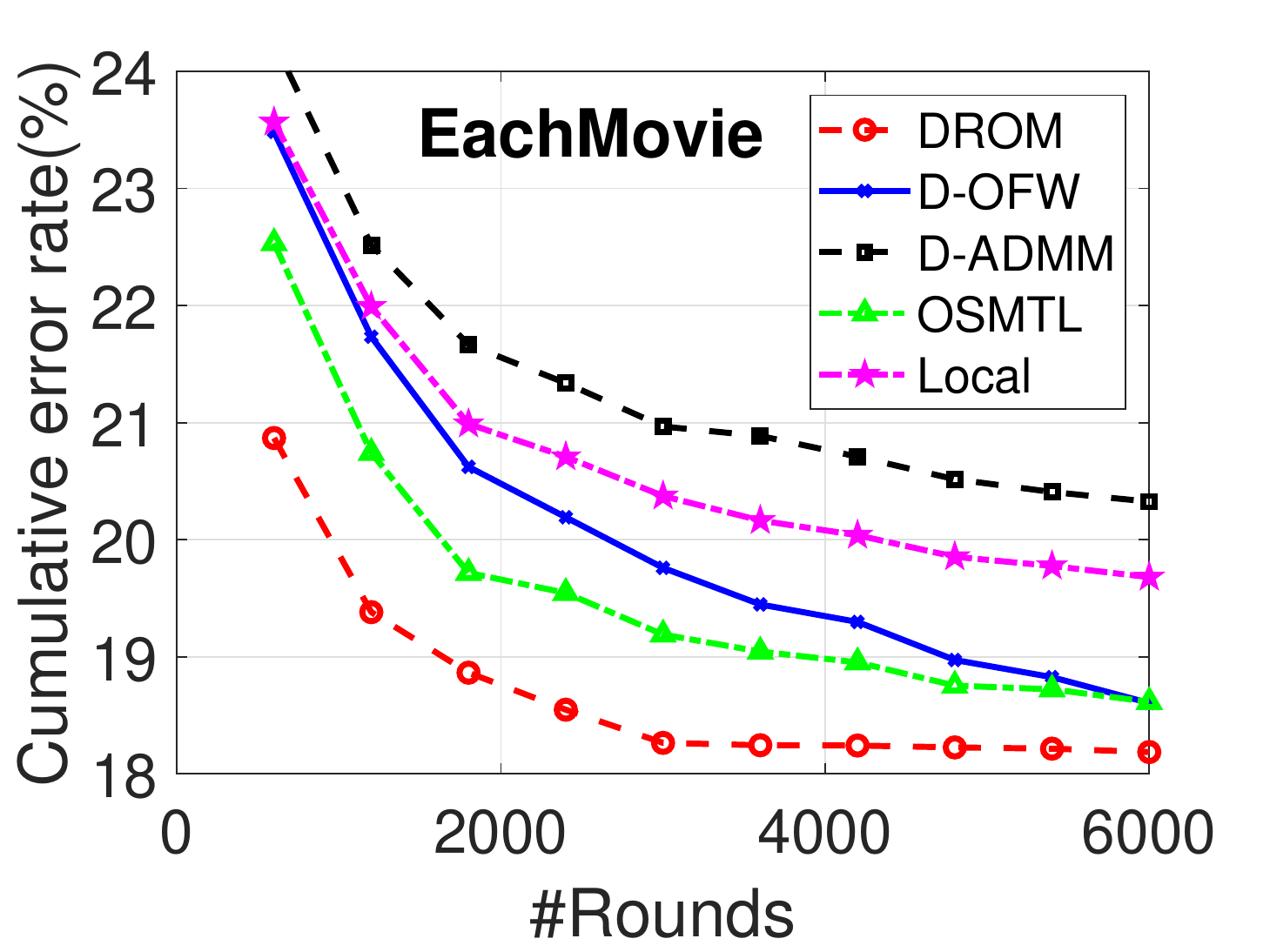}
}

\mbox{
\includegraphics[width=2.275in]{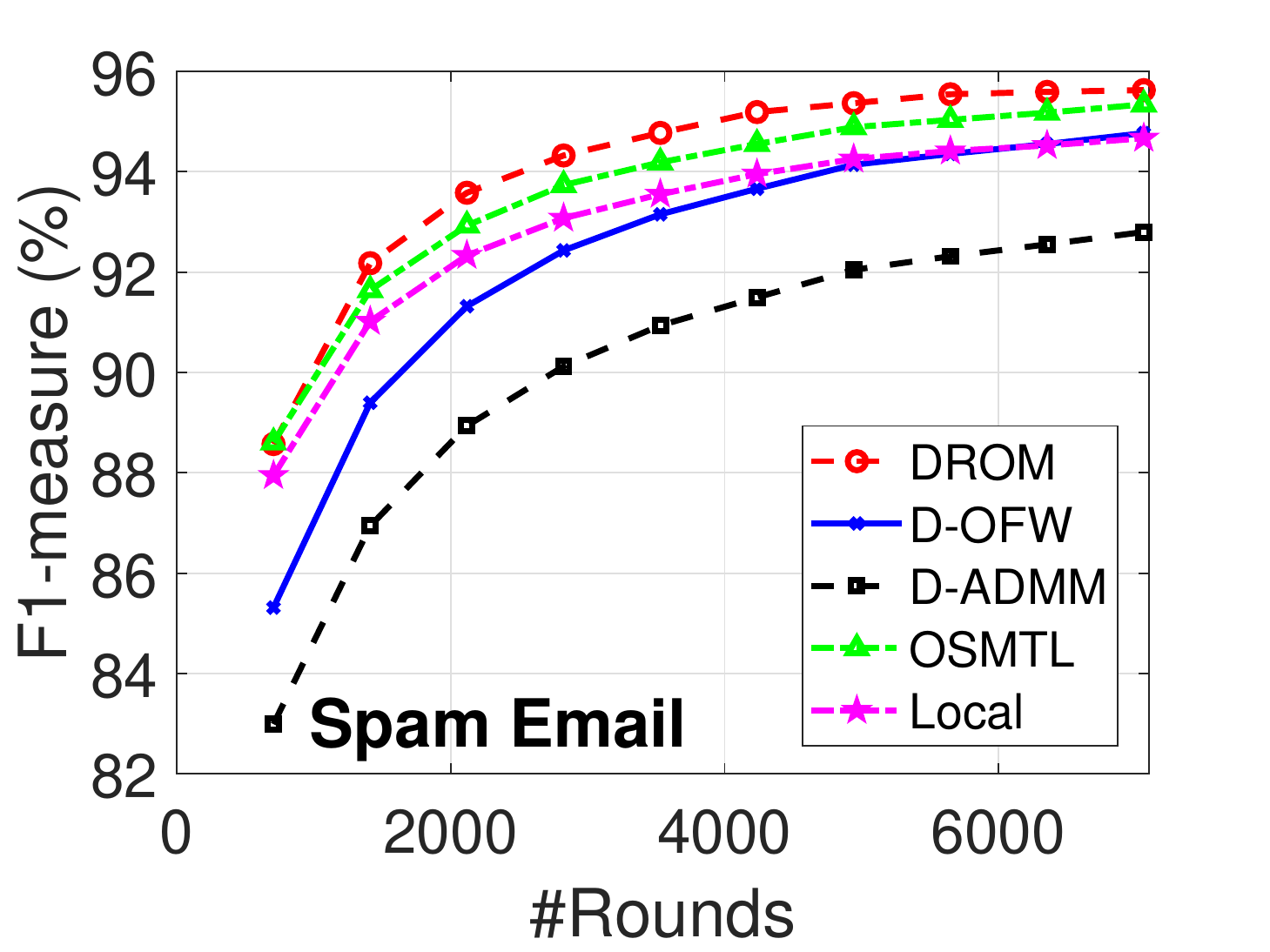}
\includegraphics[width=2.275in]{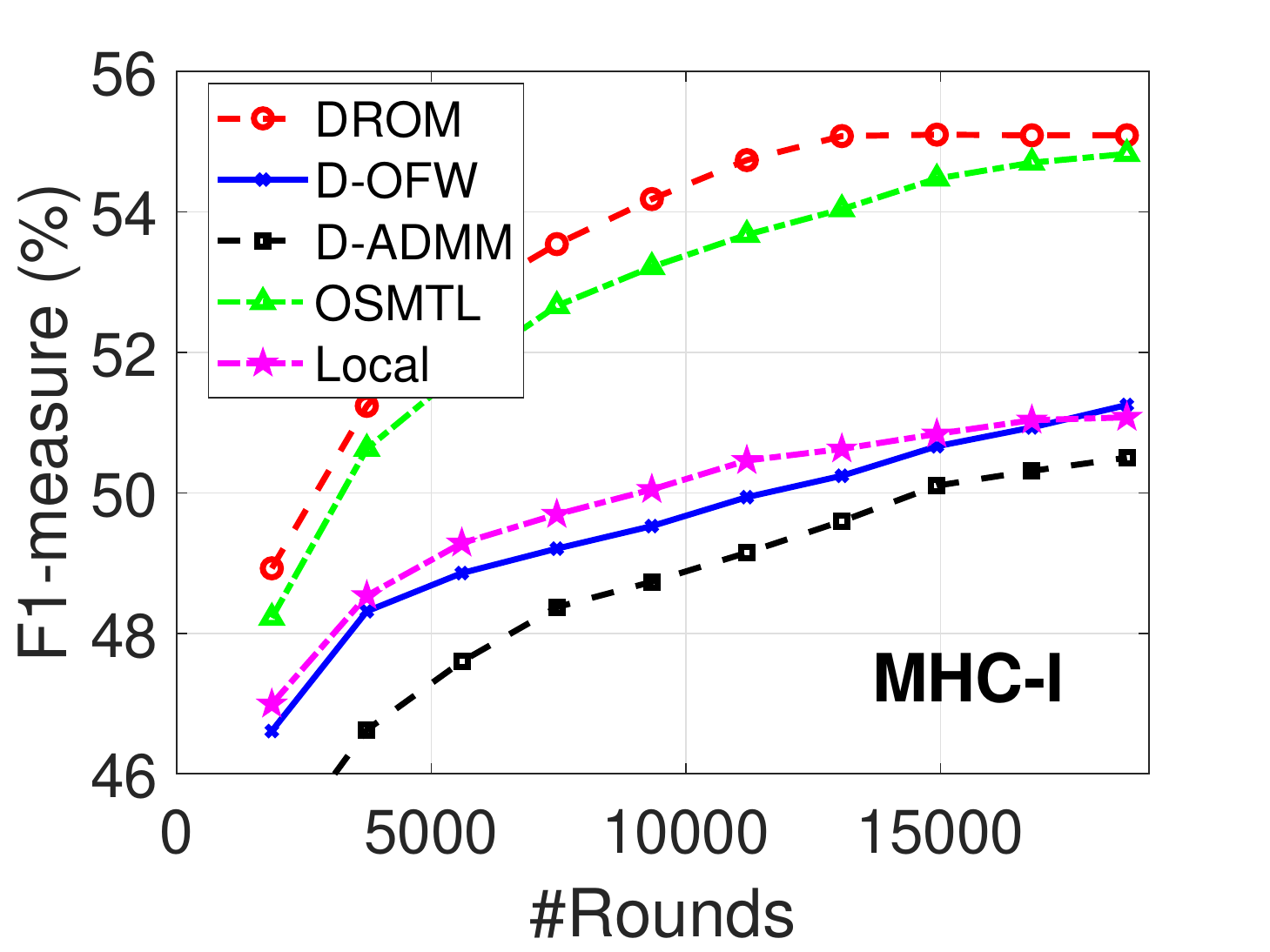}
\includegraphics[width=2.275in]{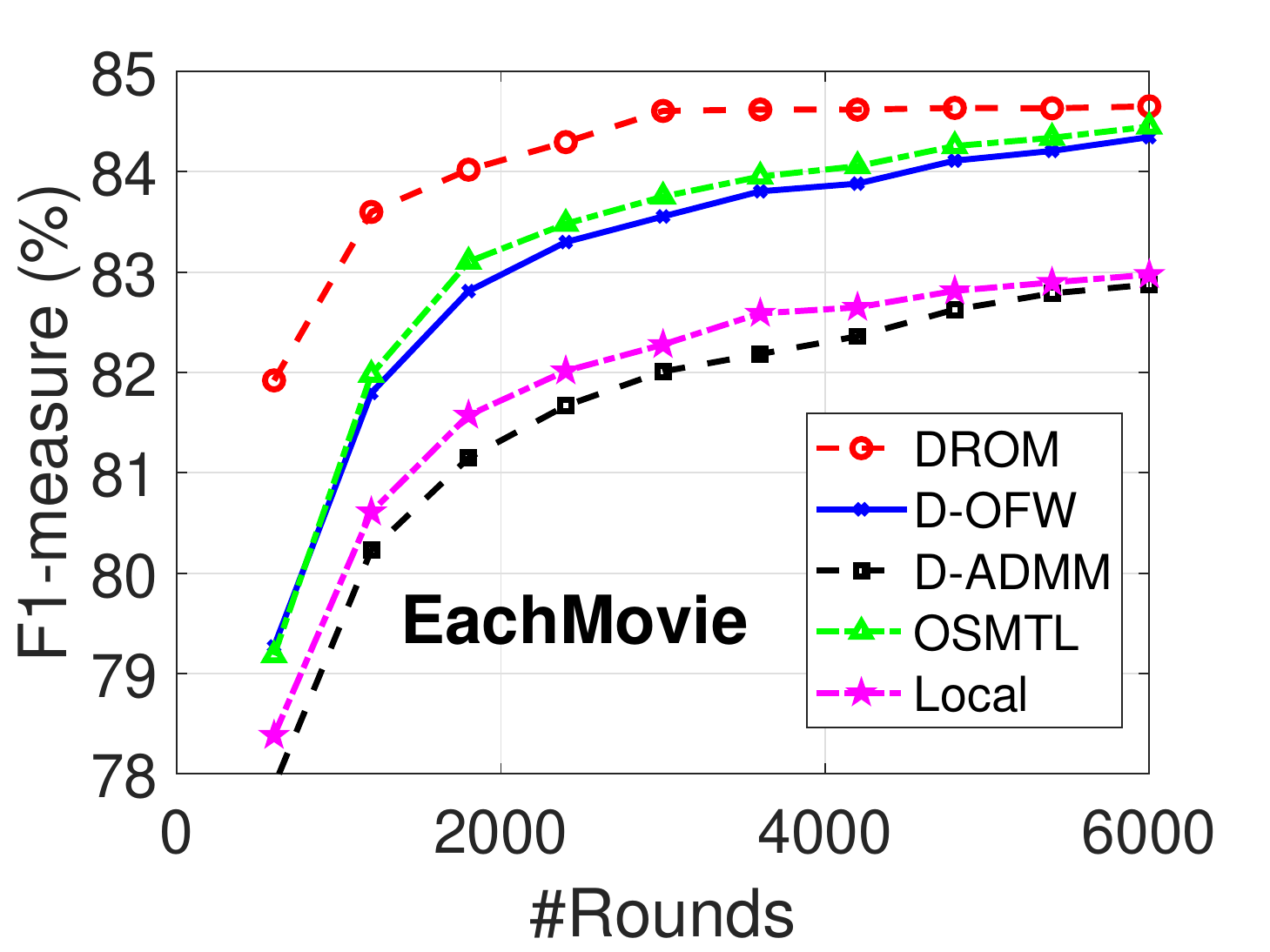}
}
\end{center}
\vspace{-0.15in}
\caption{Cumulative error rate and F1-measure along online learning process}\label{error-F1-query}
\end{figure*}

\begin{figure*}[t]
\begin{center}
\mbox{
\includegraphics[width=2.275in]{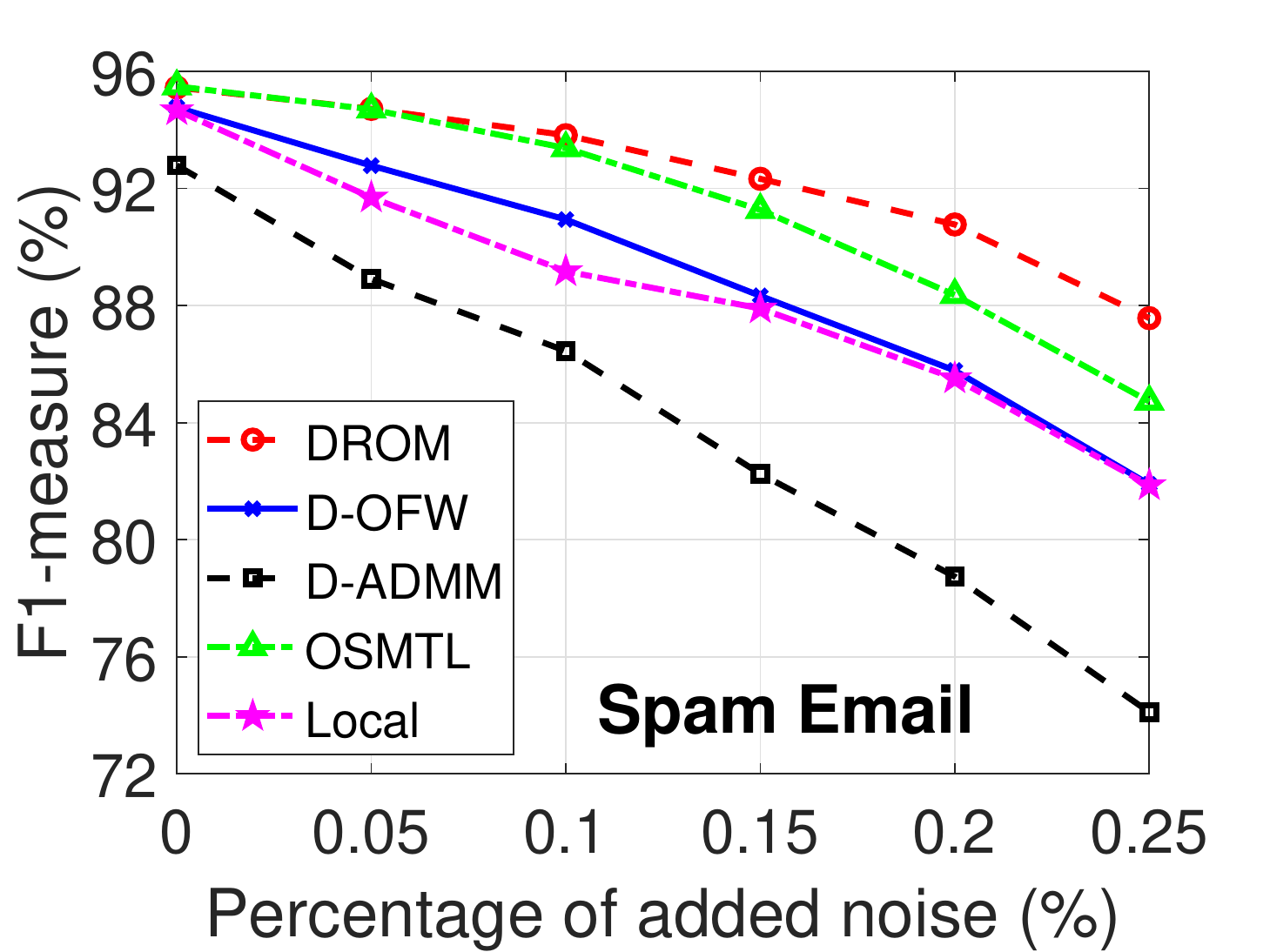}
\includegraphics[width=2.275in]{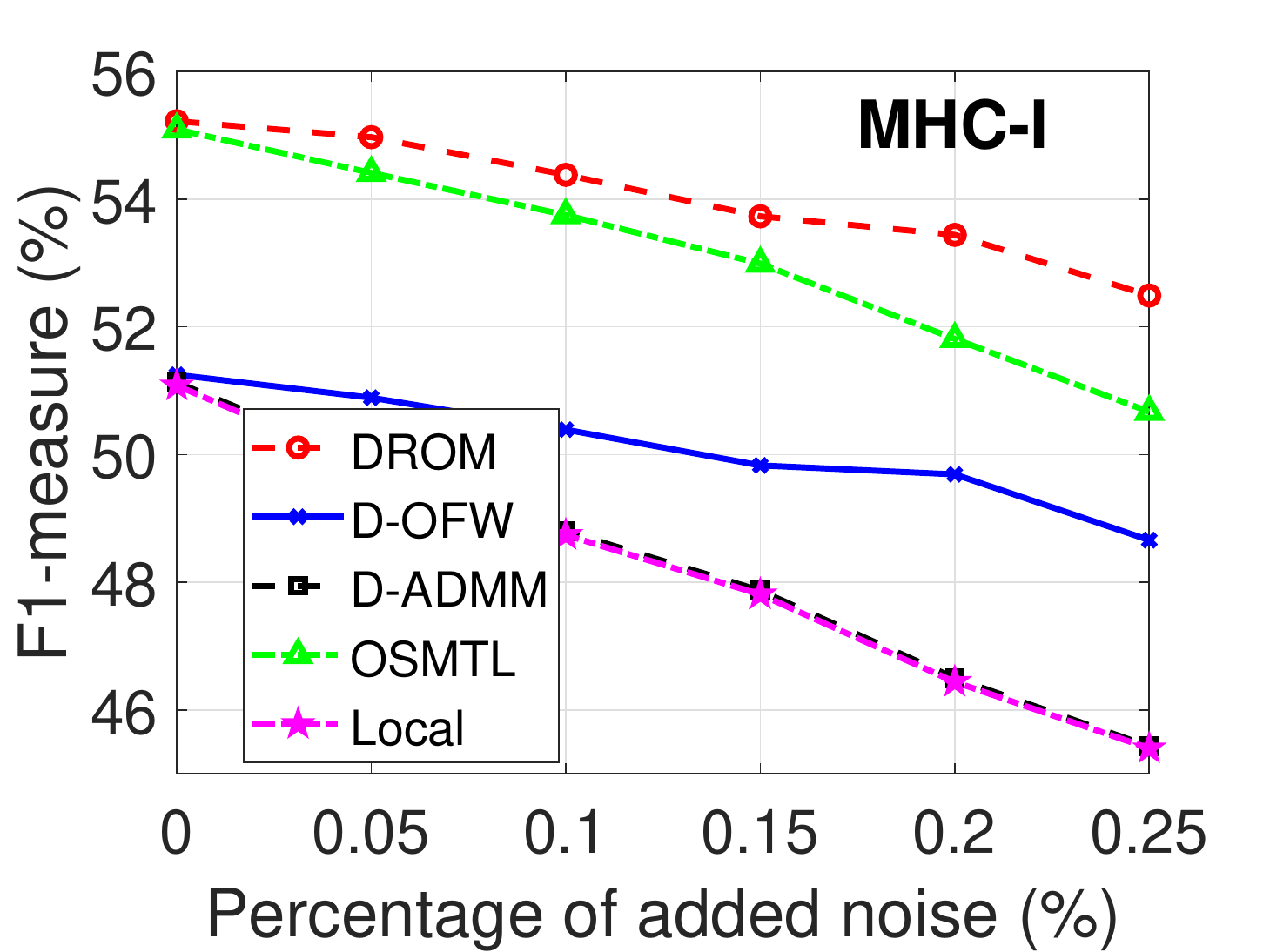}
\includegraphics[width=2.275in]{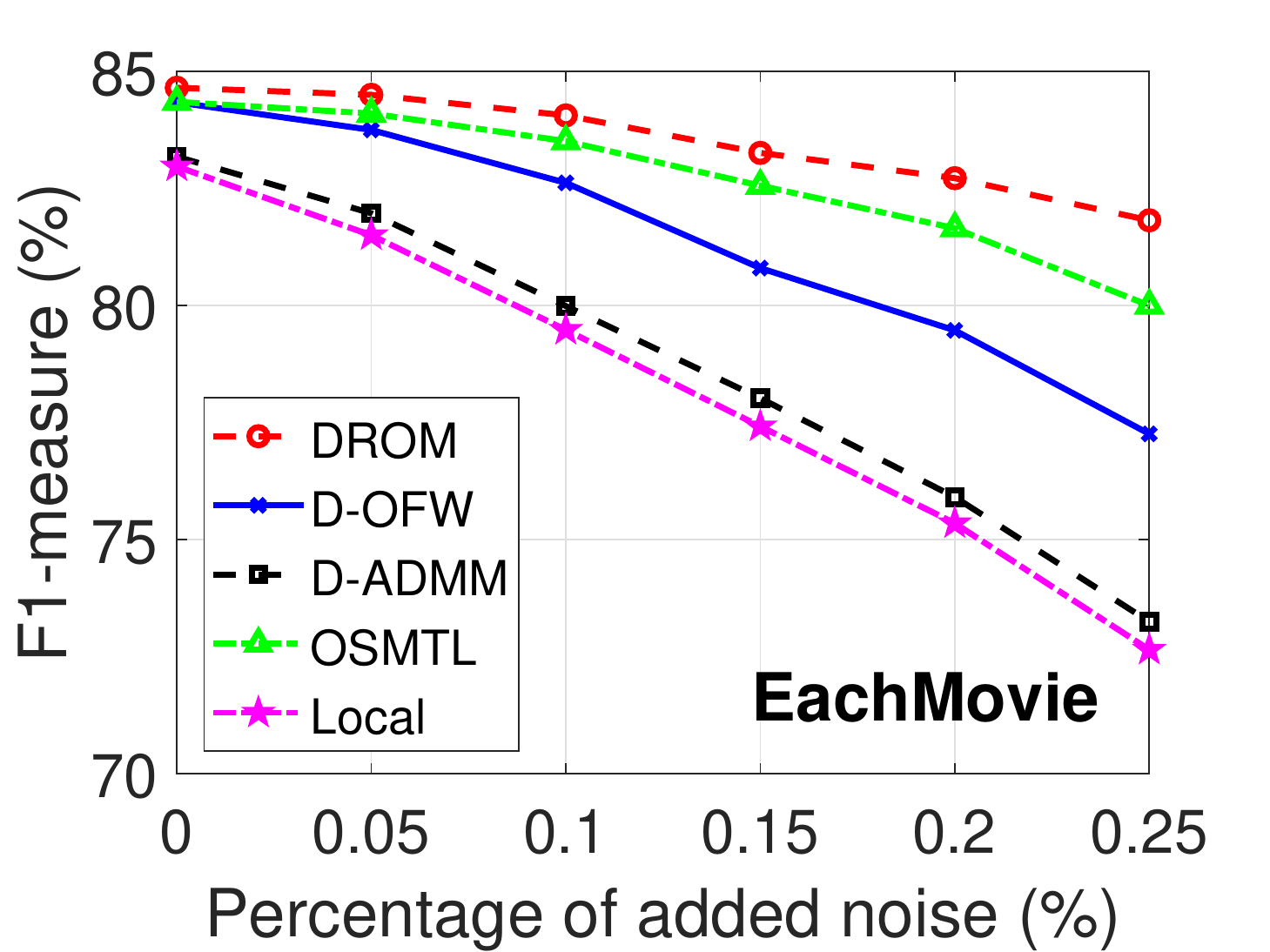}}
\end{center}
\vspace{-0.15in}
\caption{Classification accuracy performance under adversarial noisy data}\label{noisy-ratio}
\end{figure*}

\section{Experiments}

Empirical experiments are conducted to evaluate the algorithms on three  datasets used in previous work~\cite{zhang2018distributed}.
Table~\ref{statistic_data} summarizes the statistics of the datasets.\vspace{-0.2in}

\begin{table}[h!]
\centering
\caption{Description of the datasets}
\label{statistic_data}
\begin{tabular}[1.1\textwidth]{|l|r|r|r|} \hline
            & Spam Email & MHC-I & EachMovie \\ \hline
\#Tasks       & 4     & 12     & 30 \\ \hline
\#Sample      & 7,068  & 18,664  & 6,000 \\ \hline
\#Dimesion    & 1,458  & 400    & 1,783 \\ \hline
\#MaxSample   & 4,129  & 3,793   & 200   \\ \hline
\#MinSample   & 710   & 415    & 200   \\ \hline
\end{tabular}\vspace{-0.1in}
\end{table}

Spam Email\footnote{http://labs-repos.iit.demokritos.gr/skel/i-config/} contains 7,068 emails collected from mailboxes of 4 users (i.e., 4 tasks). Each mail entry is represented by a word document vector via the TF-IDF conversion technique. A classifier is proposed to classify each incoming email into two categories: \emph{legitimate} or \emph{spam} for each user.

MHC-I\footnote{http://web.cs.iastate.edu/~honavar/ailab/}, a bio-marker dataset, contains 18,664 peptide sequences for 12 MHC-I molecules (i.e., 12 tasks). Each peptide sequence is converted to a 400 dimensional feature vector~\cite{LiCHLJ11}. The learner aims to classify whether a peptide sequence is \emph{binder} or \emph{non-binder} for each MHC-I molecule. Recent work has demonstrated that the shared knowledge among related molecules (tasks) can be leveraged to improve the MHC-I binding prediction.

EachMovie\footnote{http://goldberg.berkeley.edu/jester-data/} is a movie recommendation dataset where 72,916 users rate a subset of 1,628 movies. It randomly prioritizes 6,000 user-rating pairs spanning 30 users and 200 movies. The ratings (i.e. $[1,6]$) are converted into \emph{like} or \emph{dislike}, based on the rating order. For each movie, we randomly select 1,783 users who viewed that movie and use their ratings as its features. Finally, we obtain 200 instances (1,783 features) for each of 30 tasks.


\begin{table*}[t]
\centering
\caption{Sensitivity study on the parameters $\tau$ and $\zeta$}
\label{Sensitivity-tau-zeta}
\begin{tabular}[3\textwidth] {|c|c|c|c|c|c|c|}
\hline
\multirow{2}{*}{Parameter Setting} & \multicolumn{2}{|c|}{Spam Email} & \multicolumn{2}{|c|}{MHC-I}  &  \multicolumn{2}{|c|}{EachMovie}  \\
\cline{2-7}
                             & Error Rate    & F1-measure    & Error Rate    & F1-measure    & Error Rate   & F1-measure \\ \hline\hline
$\tau=1, \zeta=0$			 & 5.31 (2.14)	 & 95.63 (1.75)  & 38.05 (0.31)	 & 55.11 (0.49)  & 18.13 (6.54) & 84.61 (8.36) \\ \hline
$\tau=1, \zeta=0.5$		     & 5.52 (3.32)	 & 95.32 (1.64)	 & 38.61 (1.21)	 & 54.51 (1.51)  & 18.90 (6.12) & 83.22 (8.77) \\ \hline
$\tau=1, \zeta=0.9$		     & 6.36 (0.64)	 & 94.67 (0.54)	 & 41.62 (3.95)	 & 51.08 (6.23) & 19.78 (7.39) & 82.97 (9.35)  \\ \hline\hline
$\tau=20, \zeta=0$			 & 5.32 (2.16)	 & 95.63 (1.76)	 & 38.13 (0.21)	 & 55.05 (0.33)	 & 18.27 (6.67) & 84.51 (8.45)  \\ \hline
$\tau=20, \zeta=0.5$		 & 5.67 (3.12)   & 95.03 (1.98)	 & 38.99 (0.54)	 & 54.11 (1.35)  & 19.35 (5.39) & 83.07 (8.99)  \\ \hline
$\tau=20, \zeta=0.9$         & 6.53 (0.43)   & 94.29 (0.36)  & 41.93 (1.97)  & 50.81 (1.23)	 & 19.86 (5.32) & 82.91 (7.33)  \\ \hline
\end{tabular}
\end{table*}

\subsection{Baselines and Evaluation Metrics}

We compare our method with four baselines:
1) \textbf{Local}, where each task learns a model locally on its own data.
2) {\bf Smoothed OMTL} (OSMTL)~\cite{murugesan2016adaptive} jointly learns the per-task hypothesis and the inter-task relationships in an online setting. 
3) Two distributed optimization methods for regularized online multi-task learning: Online Alternating Direction Method of Multipliers ({\bf D-ADMM})~\cite{matamoros2017asynchronous} and Online Frank-Wolfe ({\bf D-OFW})~\cite{zhang2017projection}. We adapt two algorithms into distributed multi-task setting, and provide corresponding implementations in Supporting Materials.
To handle with online data, we modify the offline setting of ADMM by retaining online data after observing one example. 
All parameters of the baselines are tuned according to their recommended instructions.
DROM and DORM-D are the proposed distributed algorithms.
For both methods, we simply set $\lambda = 1, \rho = 1$ to avoid overfitting, and tune $p\in(0,1)$ with $\xi = 1$ to deal with adversarial noise.

There are no good ways of unitizing network when prior knowledge of tasks is unknown. Generally speaking, there are three different types of networks:  {\em full-connected} ($\zeta = 0$), {\em rid-connected} ($\zeta = 0.5$) and {\em ring-connected} ($\zeta = 0.9$) network, used to examine the impact of adjacency matrix $\S$, where $\zeta = \max(|\sigma_2(\S)|,|\sigma_m(\S)|)$ is the second largest absolute eigenvalue of $\S$. Specifically, $S_{ij} = 1$ indicates a connection between task $i$ and task $j$; $S_{ij} = 0$ otherwise. Moreover, there are two types of synchronization, {\em fully synchronization} ($\tau = 1$) and {\em periodic synchronization} (e.g., $\tau = 20$) after every (e.g.,) $20$ rounds.\\

We evaluate the performance using two measurements:\\
1) {\bf cumulative error rate},  ratio of predicted errors over online data, reflecting the prediction accuracy of online learning; 
2) \textbf{F1-measure}, the harmonic mean of precision and recall, evaluating the performance of classification model.
Number of iteration (trial) is used to reflect the convergence of online algorithms~\cite{zhang2018distributed}, which is different from offline setting with CPU time~\cite{smith2017federated}.
For error rate, the smaller the measures, the better the performance of an algorithm; For F1-measure, a higher value means a better performance.
To compare these algorithms fairly, we randomly shuffle the ordering of samples in each dataset. We repeat each experiment 10 times and report the averaged results.

\subsection{Comparison Result}

Evaluation measures versus running rounds of online learning is plotted in Figure~\ref{error-F1-query}. The results illustrate the following:
\begin{itemize}
\item Among all the baselines, DROM achieves a lower error rate and a higher F1 score on most measures. 
\item The improvement of our algorithm over the baselines is  significant. As can be seen, our method converges faster than other baselines. This is expected as DROM achieves an optimal regret with an efficient runtime complexity.
\item  Although D-OFW has a higher order of regret, it practically obtains a better result than strong baselines.
\item Nuclear norm regularization boosts the prediction performance over plain single task learning significantly, which infers the effectiveness of leveraging the shared knowledge in multi-task learning.
\end{itemize}

To evaluate the robustness of the algorithms, we randomly impose adversarial noisy labels with a probability from $5\%$ to $25\%$. Figure~\ref{noisy-ratio} presents the evaluation measures of the algorithms on various noisy levels.
We observe that DROM consistently outperforms other methods over various levels of noise data.
This shows the clear advantage of developing robust loss functions on adversarial learning scenario.

\begin{table}[h!]
\centering
\caption{Run-time (sec) of each iteration for each algorithm}
\label{running_time}
\begin{tabular}[2.1\textwidth]{|c|c|c|c|c|} \hline
Algorithm & Spam Email & MHC-I & EachMovie \\ \hline
Local       & 0.53    & 0.76   & 1.14   \\ \hline
D-OFW         & 1.16    & 1.50   & 2.33   \\ \hline
D-ADMM        & 1.92   & 3.35   & 4.01   \\ \hline \hline
DROM        & 1.26   & 1.59   & 2.25   \\ \hline
\end{tabular}
\end{table}

We evaluate these algorithms with runtime cost in Table~\ref{running_time}. It can be observed that DROM runs faster than D-ADMM. The reason should be obvious as D-ADMM has to perform SVD in each round, while DROM computes only the leading singular vectors. DROM is relatively slower than Local, which is expected since DROM has to learn the structure of task relativeness. However, the extra computational cost is worth it as learning multiple tasks jointly can significantly improve the prediction performance.

%
%

\subsection{Sensitivity study on the parameters $\tau$ and $\zeta$}

We conduct sensitivity analysis on the parameters $\tau$ and $\zeta$. A high value of $\tau$ or $\zeta$ would reduce inter-worker communication, which gradually leads to independent learning on local tasks. Specifically, we set $\tau$ to $\{1, 20\}$ and $\zeta$ to $\{0,0.5,0.9\}$, and evaluate DROM-D in various $\mathcal{P}(\S,\tau)$. The comparison result is shown in Table~\ref{Sensitivity-tau-zeta}.
We observe that either increasing a value of $\tau$ or $\zeta$ would degrade the performance. In a fully-connected setting ($\zeta$ = 0), large synchronous interval ($\tau = 20$) is tolerant since the workers can interact with others to leverage the task relativeness. In a sparse-connected network ($\zeta > 0$), frequent synchronization ($\tau=1$) is preferable since it can accelerate propagation of information between the tasks. To achieve a balance, we choose $\tau = 20$ in fully-connected tasks in our experiment since the algorithm achieves a good accuracy with a low cost of communication.

\section{Conclusion}
\vspace{-0.02in}

This paper studies distributed primal-dual adaptive optimization for online multi-task learning. Specifically, we propose an adaptive projection-free algorithm with optimal regret and computational efficiency. Furthermore, the proposed algorithm is well-adapted in decentralized periodic-connected network with theoretical analysis based on task relatedness. We evaluate the efficacy of the proposed algorithm on three real-world datasets for multi-task classification and find out it runs significantly faster than the counterpart algorithms with projection. The theoretical results regarding the robust learning on adversarial noise have also been verified.

\vspace{-0.08in}

\bibliographystyle{aaai20}
\bibliography{standard}

\end{document}